\documentclass[letterpaper, 10 pt, conference]{ieeeconf}  

\IEEEoverridecommandlockouts                              

\usepackage{amsmath,amsfonts,amssymb,mathtools,amsthm}

\theoremstyle{plain}

\newtheorem{proposition}{Proposition}

\theoremstyle{definition}
\newtheorem{definition}{Definition}
\newtheorem{assumption}{Assumption}
\newtheorem{remark}{Remark}
\usepackage{float,subcaption}
\usepackage{multirow}
\usepackage{url}
\usepackage{bm,bbm}

\newcommand{\real}[1]{\mathbb{R}^{#1}}
\usepackage[super]{nth}

\title{\LARGE \bf
Adaptive Input Estimation in Linear Dynamical Systems with Applications to Learning-from-Observations
}

\author{Sebastian Curi, Kfir. Y. Levy and Andreas Krause
\thanks{Sebastian Curi and Andreas Krause are with the Learning \& Adaptive Systems Group, Department of Computer Science, ETH Zurich, Switzerland. Kfir. Y. Levy is with the Electrical Engineering Department, Technion-Israel Institute of Technology.
        Emails: {\tt\small scuri@inf.ethz.ch, kfirylevy@technion.ac.il and krausea@ethz.ch}.%
}}

\begin{document}

\maketitle
\thispagestyle{empty}
\pagestyle{empty}

\begin{abstract}

We address the problem of estimating the inputs of a dynamical system from measurements of the system's outputs. 
To this end, we introduce a novel estimation algorithm that explicitly trades off bias and variance to optimally reduce the overall estimation error.  
This optimal trade-off is done efficiently and adaptively in every time step. 
Experimentally, we show that our method often produces estimates with substantially lower error compared to the state-of-the-art.
Finally, we consider the more complex \emph{Learning-from-Observations} framework, where an agent should learn a controller from the outputs of an expert's demonstration. 
We incorporate our estimation algorithm as a building block inside this framework and show that it enables learning controllers successfully.

\end{abstract}

\section{INTRODUCTION}

Input estimation consists of finding a sequence of inputs that produced a sequence of outputs for a given dynamical system.
It has found many application in practice, such as fault detection and isolation \cite{chen1996optimal}, vehicle tracking \cite{chan1979kalman}, sensor bias estimation \cite{park2004dynamic}, economics \cite{li2013state}, and geophysics \cite{kitanidis1987unbiased}. 

Recently, a similar problem received attention in the Imitation Learning community. 
The \emph{Learning-from-Observations} (LfO) problem, introduced in \cite{liu2017imitation}, consists of learning a controller from the outputs of a system when an expert controls it. 
It is more challenging than standard imitation learning methods that allow access to input-output pairs \cite{hussein2017imitation}. 

The motivation of the LfO framework is to learn controllers from everyday data instead of controlled experiments.
For example, imitation learning of self-driving cars requires an expert to drive a car where both observations and driver inputs are recorded. 
On the other hand, there exists a lot of video footage of cars driven by drivers in daily life (e.g., from surveillance cameras) where no direct input is available. 
The LfO framework is a natural way of modeling this problem. 

As one can imagine, an approach to solving the LfO problem uses \emph{input estimation} methods to yield estimates of the actual inputs \cite{torabi2018behavioral,edwards2018imitating}. 
Then, it uses (recorded observations, \emph{estimated} inputs) pairs inside standard imitation learning methods.
The quality of the controllers learned this way crucially depends on the quality of the input estimates.
Thus, in this work, we focus on the input estimation problem.
Concretely, we consider the setting where a teacher controls a given dynamical system.
A learner has access to measurements of the system's outputs, and her goal is to estimate the teacher's inputs from these measurements. 

\textbf{Our Contributions:}
We design a novel and efficient algorithm called Adaptive Linear Input Estimator (AdaL-IE) for the input estimation problem.
Our algorithm predicts the input at each time step in a manner that optimally trades off bias and variance, to yield low error estimates.
Moreover, our method provides bounds on the estimation error. 
In contrast to other approaches, our method does not require any prior knowledge about the process or measurement noise, such as Gaussianity, but only requires a bound on the magnitude (or the variance) of these noise terms.
Instead, we exploit the structure on the input signal, which we assume to be Lipschitz continuous. 
Although stringent at first, the low-pass characteristics of actuators make this assumption to hold in practice. 
Furthermore, we do not need to know the Lipschitz constant nor the noise bound, but just the Lipschitz-to-Noise ratio that has the same interpretation as the conventional Signal-to-Noise ratio. 
In the case that this ratio is unknown, a simple cross-validation scheme can be used to find this ratio. 
We also quantify the effect of system uncertainty on our predictions, and we show how to extend our approach to address non-linear systems.
Finally, we apply our method within the LfO framework to an inverted pendulum. 
We show that it enables learning stabilizing controllers while the state-of-the-art input estimation often fails to stabilize the system. 

\subsection{Related Work}
Input observability, as well as conditions on the system that ensure this property, were introduced in \cite{hou1998input}. 
However, the authors did not provide any practical method to estimate the inputs. 
Later, practical methods were designed based on Kalman filtering techniques using joint or two-step state-input estimation. 
One such method, which is considered to be the \emph{state-of-the-art} input estimator, is the Unbiased Minimum Variance Input Estimator (UMV-IE) \cite{gillijns2007unbiased}. 
As the name suggests, this method ensures unbiased input estimates. 
However, UMV-IE often suffers from high variance, is only optimal under particular cases, such as Gaussian noise and linear dynamical systems, and may become unstable if not tuned appropriately \cite{hsieh2014implementation}.
These claims also apply to other methods based on Kalman filters ~\cite{hsieh2010optimality,yong2013unified}. 

An alternative to Kalman filtering was presented in \cite{park2000estimation}, where authors invert the system and truncate the impulse response to the first $M$ components to find the "inverse" dynamical system by making the strong assumption that the unknown input sequence is generated by a random walk and the common assumptions that the system and initial state distribution, as well as the mean and covariances of the measurement and process noise, are known
The resulting estimator is a \emph{fixed} weighted linear combination of the past $M$ measurements. 
As shown in this paper, the choice of $M$ is crucial for the performance of their estimator. 
Nevertheless, the authors do not discuss how to choose $M$ appropriately.

In \cite{liu2017imitation}, authors introduced the LfO framework and solved it using an off-the-shelf reinforcement learning algorithm where the per-step cost function is the distance between the observations of the expert and the learned trajectories.  
This work suffers from high sample complexity, hence in \cite{torabi2018behavioral} and \cite{edwards2018imitating}, authors extended this framework to behavioral cloning imitation learning \cite{bain1999framework}.
In these works, authors consider discrete states and actions on a Markov Decision Process (MDP). 
The algorithm they propose iterates between learning the inverse model of the MDP around the expert trajectory and filtering the observed states with the learned (inverse) model to estimate the inputs. 
First, the learner applies the estimated input to the real system and updates the model. 
Second, the learner uses the model to re-estimate the input.  

In contrast to Kalman based input estimators, our approach does not enforce unbiased estimates. 
Thus, we obtain a lower estimation error, which is crucial for the LfO problem.
Furthermore, while the Kalman based estimators optimality crucially relies on the Gaussianity assumption, our method enables us to handle much more general noise terms.
Compared to the truncated inverse method, we do not assume that the input is generated with a random walk, but rather that the input is Lipschitz continuous. 
Moreover, we explicitly select the window width with a principled optimization procedure that trades off bias with variance. 
For the LfO problem, our method may be used as a building block that allows generalizing discrete to continuous state and actions.

Finally, we comment that all of the existing works on input estimation assume the knowledge of the system dynamics. 
This assumption is necessary since the case where both the system and inputs are unknown is an ill-posed problem.
In this context, we show how a mismatch between our model and the \emph{true} system introduces extra bias in our estimation.  

\section{BACKGROUND AND PROBLEM STATEMENT} \label{sec:background}
\subsection{Mathematical Notation}  
We denote a vector at time $t$ by $v_t \in \real{n}$, a sequence of vectors by $\{ v_{t_1}, v_{t_1+1}, \ldots, v_{t_2-1}, v_{t_2} \} = \{v_\tau \}_{\tau=t_1}^{t_2}$, and its $i$-th entry by $v_t[i]\in \real{}$. 
The Moore-Penrose pseudo-inverse of a matrix $A \in \real{n\times m}$ is $A^\dagger =  (A^\top A)^{-1} A^\top$.
We use $\| \cdot \|$ for the 2-norm of a vector or the induced 2-norm of a matrix. 
We denote the identity matrix of size $n$ by $I_n$ and the delta dirac function by $\delta(\cdot)$. 

\subsection{Input Estimation in Linear Dynamical Systems} \label{sec:DynamicSystem}
A time-varying discrete-time linear dynamical system \cite{dahleh2004lectures} can be described with a state $x_t \in \real{n_x}$, and the input $u_t \in \real{n_u}$ to output $y_t \in \real{n_y}$ relationship evolves as:
\begin{equation}
  \left\{ \begin{aligned}
  x_{t+1} &= A_t x_{t} + B_t u_{t} + \epsilon_t \\
  y_{t} &= C_t x_{t} + \nu_t,\end{aligned} \right. \label{eq:linsys}
\end{equation}
where $\epsilon_t \in \real{n_x}$ is process noise and $\nu_t \in \real{n_y}$ is measurement noise. 
The state transition matrix is defined as: $\Phi_{(t, i)} \equiv A_{t-1} \Phi_{(t-1, i)}$, for $0 \leq i < t$, and $\Phi_{(t, t)} = I$.
The output of the system at time $t$ is expressed as: 
\begin{equation}
y_t = C_t\left(\Phi_{(t, 0)} x_0 + \sum_{i=0}^{t-1} \Phi_{(t, i+1)} (B_i u_i+\epsilon_i)\right) + \nu_t. \label{eq:output}
\end{equation}

\begin{definition}[Strong Observability \cite{yong2013unified}] \label{def:strong_observability}
The linear system \eqref{eq:linsys} is considered \emph{strongly observable}, or equivalently state and input observable, if the initial condition $x_0$ and the unknown input sequence $\{u_\tau \}_{\tau=0}^{r-1}$ can be uniquely determined from the measured output sequence $\{y_\tau \}_{\tau=0}^r$, for a large enough number of observations. 
\end{definition}

\begin{assumption}[System Knowledge] \label{assum:system_knowledge}
The system dynamics is known and it is strongly observable. 
Furthermore, we have an unbiased estimate $\hat{x}_0$ of the initial state, $x_0 \equiv \hat{x}_0 + \epsilon_0$.
\end{assumption}
\begin{remark}
System knowledge could be acquired with physical modeling or system identification. 
Input observability is needed for well-possessedness of the input estimation problem. 
By Definition ~\ref{def:strong_observability}, the initial state $x_0$ can be reconstructed after $r$ measurements hence the initial state knowledge is a reasonable assumption in a strongly observable system.
\end{remark}

\begin{assumption}[Input Functions] \label{assum:input_functions}
We will consider the inputs $u_{t} = f(t, y_t)$ to be an unknown function of the time step $t$ and the output $y_t$. 
This models both open-loop and closed-loop behavior. 
Furthermore, the input function is Lipschitz continuous: 
For any $(t_1, t_2) \in \real{} $ it holds that $ \left\| u_{t_1} - u_{t_2} \right\| = \|f(t_1, y_{t_1}) - f(t_2, y_{t_2}) \| \leq L d(t_1, y_{t_1}; t_2, y_{t_2})$, where the distance function $d(\cdot;\cdot)$ is known in advance.
\end{assumption}
\begin{remark}
Assumption~\ref{assum:input_functions} is stringent but the low-pass characteristic of actuators makes this assumption to hold in practice. 
It is also less restrictive than random-walk assumptions done in previous work. 
The distance function is problem dependent and key for the success of this method. 
In this work, we use $d(t_1, y_{t_1}; t_2, y_{t_2}) = |t_2 - t_1| + \|y_{t_2} - y_{t_1}\|_2 $. 
\end{remark}

\begin{assumption}[System Noise] \label{assum:noise}
The process $\epsilon_t$ and measurement $\nu_t$ noise terms are zero-mean and bounded by $b$. 
The noise sequence is i.i.d.~and uncorrelated, $\mathbb{E}[\epsilon_t \epsilon_{t'}] = \mathbb{E}[\nu_t \nu_{t'}] = 0 $, for all $t \neq t'$, and $\mathbb{E}[\epsilon_t \nu_{t'}] = 0$, for all $t, t'$.
\end{assumption}
\begin{remark}
We need Assumption~\ref{assum:noise} to apply Hoeffding's inequality when deriving an upper bound of the estimation error. 
Alternatively, we could assume that the variance of the noise is bounded and apply Bernstein inequalities -- the result and analysis remain qualitatively similar. 
If the noise is not zero-mean, then the mean is included in the input estimation.
\end{remark}

\begin{remark}
Without loss of generality one can consider that $B_t = I_{n_u}$ and solve the input estimation problem for the input $\tilde{u}_t = B_t u_t$. 
The input is estimated as $\hat{u}_t = B_t^\dagger \hat{\tilde{u}}_t$. 
Furthermore, if the original system is observable, the modified system is strongly observable and $\tilde{u}_t$ can be estimated in any observable system.
\end{remark}

\section{ADAPTIVE LINEAR INPUT ESTIMATOR} \label{sec:input_estimator}

\subsection{Estimator Architecture}
The unbiased, minimum variance input estimator (UMV-IE) has a recursive and linear architecture that ensures unbiasedness \cite[Theorem 1]{gillijns2007unbiased}. 
To have a lower estimation error, we allow our estimates to break the unbiasedness constraint and propose for each timestep $t$, the following anti-causal finite impulse response estimation architecture:
\begin{equation}
    \hat{u}_t = \sum_{\tau=1}^{T} \alpha_t[\tau] \left(C_{\tau} H_\tau \right)^\dagger \left(y_\tau - C_\tau \Phi_{(\tau, 0)} \hat{x}_0\right), \label{eq:estimator}
\end{equation}
where $C_\tau H_\tau \equiv \sum_{i=0}^{\tau-1} C_\tau \Phi_{(\tau, i+1)}$ is the $\tau$-row sum of the invertibility matrix $\mathcal{I}_T$, and the filter parameters $\alpha_t$ are defined over the simplex, $\Delta_T = \{ \alpha_t \in \real{T} | \sum_{\tau=1}^T \alpha_t[\tau] = 1 \text{ and } \alpha_t[\tau] \geq 0,\;\forall \tau \in \{1, 2, \ldots, T\} \}$. 
By Assumption ~\ref{assum:system_knowledge}, the product $ \left(C_{\tau} H_\tau \right)^\dagger  \left(C_{\tau} H_\tau \right) = I$, else $\mathcal{I}_T$ would be rank deficient and the system not strongly observable. Furthermore, the inverse of $H_\tau$ exists and is unique. 

Next, we give some intuition behind the structure of the estimators in \eqref{eq:estimator}. 
For a constant input $u$ and when no noise is present, $u = \left(C_{\tau} H_\tau \right)^\dagger \left(y_\tau - C_\tau \Phi_{(\tau, 0)} \hat{x}_0\right)$, $\forall \tau$. 
When the input is constant, but noise is present, then averaging the different $\tau$-indexes with weights $\alpha_t \in \Delta_T$ concentrates the different noise components to average it out. 
However, when the input is not constant, and noise is present, then averaging the different $\tau$-indexes with weights $\alpha_t \in \Delta_T$ will average out the noise, but will also introduce a bias in the estimation, as discussed in the next subsection. 
The challenge resides in choosing the weights $\alpha_t \in \Delta_T$ to optimally trade-off the noise variance with the estimation bias. 

\subsection{Optimization Algorithm}
Given the estimator architecture, our goal is to find the best coefficients $\alpha_t$, that minimize the squared 2-norm of the estimation error
\begin{equation}
  \min_{\alpha_t \in \Delta_T} \left \| \sum_{\tau=1}^{T} \alpha_t[\tau] \left(C_{\tau} H_\tau \right)^\dagger \left(y_\tau - C_{\tau} \Phi_{(\tau, 0)} \hat{x}_0\right) - u_t \right \|^2.  \label{eq:min_error} 
\end{equation}
    
Although this problem is convex on $\alpha_t$ and looks like a typical regression problem, we do not have the targets $u_t$. 
To bypass this problem, we build on the analysis by \cite{anava2016k} and decompose our objective in \eqref{eq:min_error} into bias and variance terms to find a relaxed objective that is tractable to optimize.

\begin{proposition} \label{prop:relaxation}
  The optimization problem in \eqref{eq:min_error} can be upper bounded with probability at least $1-\beta$ by
  \begin{equation}
  2V(\beta) \min_{\alpha_t \in \Delta_T}  \alpha_t^\top \left( Q + \frac{L^2}{V(\beta)} q_t q_t^\top \right) \alpha_t, \label{eq:relaxed_optimization} 
  \end{equation}
  where $Q \in \real{T \times T}$ is a positive definite matrix that depends on the system dynamics, $q_t \in \real{T}$ depends on the system dynamics and the distance metric, $V(\beta) \in \real{}_{>0}$ depends on the noise magnitude and the probability level $\beta$, and $L \in \real{}_{\geq 0}$ is the input function Lipschitz constant. 
\end{proposition}

\begin{proof}
  The error can be upper bounded using in bias and variance terms:
  \begin{align*}
    \text{error}^2 =& \left \| \sum_{\tau=1}^{T} \alpha_t[\tau] \left(C_{\tau} H_\tau \right)^\dagger \left(y_\tau - C_{\tau} \Phi_{(\tau, 0)} \hat{x}_0\right) - u_t \right \|^2 \\
    \leq& \underbrace{2 \left\| \sum_{\tau=1}^{T} \alpha_t[\tau] H_\tau^{-1}\left(\sum_{i=0}^{\tau} \Phi_{(\tau, i)} \epsilon_i + \nu_\tau \right) \right\|^2}_{\text{variance}} \\
    & +  \underbrace{2 \left\| \sum_{\tau=1}^{T} \alpha_t[\tau] H_\tau^{-1}\left( \sum_{i=0}^{\tau-1} \Phi_{(\tau, i+1)} (u_i- u_t) \right) \right\|^2}_{\text{bias}^2}.
  \end{align*}
  The variance term only depends on the process and measurement noise, while the bias is present even in the noiseless case. 
  Using Assumption~\ref{assum:input_functions}, we bound the bias term:
  \begin{align*}
    \text{bias} &= \left\| \sum_{\tau=1}^{T} \alpha_t[\tau] H_\tau^{-1}\left( \sum_{i=0}^{\tau-1} \Phi_{(\tau, i+1)} (u_i- u_t) \right) \right\| \\
    & \leq \sum_{\tau=1}^{T} |\alpha_t[\tau]| \left( \sum_{i=0}^{\tau-1} \|H_\tau^{-1} \Phi_{(\tau, i+1)} \| \|(u_i- u_t)\| \right) \\ 
    & \leq L \sum_{\tau=1}^{T} \alpha_t[\tau] \left( \sum_{i=0}^{\tau-1} d(i, t) \|H_\tau^{-1} \Phi_{(\tau, i+1)} \| \right) = L \alpha_t^\top q_t.
  \end{align*} 
  where $q_t[\tau] =  \left( \sum_{i=0}^{\tau-1} d(i, t) \|H_\tau^{-1} \Phi_{(\tau, i+1)} \| \right)$. Using Hoeffding's inequality for matrix concentrations \cite{juditsky2008large,kakade2010matrixConc}, we bound the variance term with probability at least $(1-\beta)$:
  \begin{align*}
    &\text{var} = \left\| \sum_{\tau=1}^{T} \alpha_t[\tau] H_\tau^{-1}\left(\sum_{i=0}^{\tau} \Phi_{(\tau, i)} \epsilon_i + \nu_t \right) \right\|^2 \\
    &= \left\| \sum_{i=0}^{T} \sum_{\tau = \max\{1, i\}}^{T} \alpha_t[\tau] H_\tau^{-1} \Phi_{(\tau, i)} \epsilon_i + \sum_{\tau=1}^{T} \alpha_t[\tau] H_\tau^{-1} \nu_t \right\|^2 \\
    &\leq 36b^2\left(1 + \sqrt{\log(\beta^{-1})}\right)^2 \sum_{i=0}^{2T} M_i^2 \leq V(\beta) \alpha_t^\top Q \alpha_t,
  \end{align*}
  where $M_i \geq ||\sum_{\tau = \max\{1, i\}}^{T} \alpha_t[\tau] H_\tau^{-1} \Phi_{(\tau, i)}||$ for $i \in \{0, \ldots, T \}$ and  $M_i \geq ||\alpha_t[i-T] H_{i-T}^{-1}||$ for $i \in \{T+1, \ldots, 2T\}$. Defining $V \equiv 36b^2 \left(1 + \sqrt{\log(\beta^{-1})}\right)^2$ the quadratic form with $Q[\tau, \tau'] = \sum_{i=0}^{\min \{\tau, \tau'\}}  \| H_\tau^{-1} \Phi(\tau,i) \| \| H_{\tau'}^{-1} \Phi(\tau',i)\| + \left\| H_\tau^{-1} \right\|^2 \delta(\tau - \tau')$ appears. 
  The rest of the proof is simple algebra. 
\end{proof}  

The relaxed optimization problem \eqref{eq:relaxed_optimization} is a quadratic program over the simplex, which can be solved efficiently via projected gradient descent \cite{boyd2004convex,duchi2008efficient}. 
The solution of \eqref{eq:relaxed_optimization} are the $T$ parameters of the estimator.  
Furthermore, the estimator has only one hyperparameter, the Lipschitz-to-Noise ratio $\frac{L^2}{V(\beta)}$, which can be tuned by cross-validation. 
Finally, the projection onto the simplex of the $T$ parameters implicitly regularize them.  
  
\subsection{Extensions to Uncertain and Non-Linear Systems} \label{sec:extensions}
\textbf{Uncertain Systems:}
For simplicity, we consider $B_t=C_t=I_n$. If the state transition matrix estimate has an error as $A_t= \hat{A}_t + \Delta_{A_t}$, the dynamics can be expressed as:
\begin{equation*}
    \left\{ \begin{aligned}
      x_{t+1} &= A_t x_{t} + u_{t} + \epsilon_t = \hat{A}_t x_{t} + \left(\Delta_{A_t} x_t + u_t\right) + \epsilon_t \\
      y_{t} &= x_{t} + \nu_t,
    \end{aligned} \right. 
\end{equation*}
We call $\tilde{u}_t = \Delta_{A_t} x_t + u_t$ and use AdaL-IE to recover $\tilde{u}_t$. 
The Lipschitz constant of $\tilde{u}$ is $\tilde{L} = \| \Delta_{A_t}  \| + L$. 
If the mismatch is large, then $\|\Delta_{A_t}\|$ dominates the signal information of $u$, the upper bound of the bias is loose, and the estimate poor.

\textbf{Initial State Knowledge:}
If the initial state of the system is not known, then the estimation will have an extra bias of $\sum_{\tau=1}^{T} \alpha_t[\tau] \left(C_{\tau} H_\tau \right)^\dagger \Phi_{(\tau, 0)} \hat{x}_0$. 
For stable systems, this term is significant at times $t$ close to 0. 
For unstable systems, this term can be the prominent source of error for larger times. 

\textbf{Non-Linear Dynamical Systems:}
We consider a linearization of a class of non-linear systems that include most robotics systems \cite{theodorou2010generalized}:
\begin{equation*}
    \left\{ \begin{aligned}
    x_{t+1} &= g(x_t) + h(x_t) u_t+ \epsilon_t \\
    y_{t} &= C_tx_{t} + \nu_t,\end{aligned} \right. \approx 
    \left\{ \begin{aligned}
    x_{t+1} &= A_t x_{t} + B_t u_{t} + \epsilon_t \\
    y_{t} &= C_t x_{t} + \nu_t,\end{aligned} \right.
\end{equation*}
where $g(x_t)$ and $h(x_t)$ are non-linearities that depend only on the current state. 
The inputs affect the next state linearly, and the observations are linear functions of the states. 
Up to small errors, the linearization around the state trajectory yields $A_t = \left. \frac{\partial g(x)}{\partial x}\right|_{x=x_t}$, $B_t =h(x_t) +  \left. \frac{\partial h(x)}{\partial x} x\right|_{x=x_t}$.

Of course, this is not a trivial task because we do not have access to the state. 
In the non-linear systems used in this work, we assumed that $C_t$ is full rank and thus it could be inverted to obtain a \emph{noisy} estimate of the state and we use this estimate to linearize the system. 
Although stringent, this assumption is used in the literature of non-linear observers with unknown inputs \cite{ha2004state,corless1998state}. 
In fact, for the problem of non-linear state estimation with unknown inputs, few results exist for particular classes of systems.

\section{EXPERIMENTS} \label{sec:experiments}
\subsection{Input Estimation} \label{ssec:experiments:input_estimation}
We perform a set of experiments to compare AdaL-IE with UMV-IE on the RMS estimation error.
We apply a set of 8 different input signals to 5 different systems for 20 episodes of 100 time-steps. 
Even though the input signals were a priori known, we fit the SNR ratio for setting up the optimization problem \eqref{eq:relaxed_optimization} via cross-validation. 
To simulate modeling mismatch, we perturb the entries of the state transition matrix of each system with Gaussian noise. 
The signals are (a) a unit step, (b) a unit sine, (c) a unit ramp, (d) a unit triangle, (e) a ramp followed by a unit step, (f) a ramp followed by a zero step, (g) a unit step followed by a zero step, and (h) a unit sine superposed to a unit step. The systems were (i) a spring-mass system, (ii) a double integrator, (iii) a email server \cite[Section 2.6.1]{hellerstein2004feedback}, (iv) a HTTP server \cite[Section 7.8.1]{hellerstein2004feedback}, (v) a stable system with random coefficients in the state transition matrix, (vi) a first order non-linear system \cite[Example 4.27]{khalil2002nonlinear}, and (vii) a second order non-linear system \cite[Example 4.61]{khalil2002nonlinear}. 

We show the input estimation of system (i) in Fig.~\ref{fig:IE_SpringMass}.   
We plot the input signal in blue, the UMV-IE estimates mean plus/minus one standard deviation in shaded green, and AdaL-IE estimates in shaded red. 
AdaL-IE has lower error than UMV-IE due to a much lower variance, but at the cost of some bias. 
The rest of the systems are summarized in with the experimental mean RMS estimation error in Table~\ref{tab:experimenterror}. 
The top sub-row shows the AdaL-IE RMS error and the bottom one the UMV-IE RMS error.
From the 56 experiments, in 50 AdaL-IE outperforms UMV-IE, while it underperforms in 3. In the remaining 3, there is no statistical difference.
    
{\vspace{-0.2em}
\begin{figure}[htpb]
\scriptsize
  \centering
  \begin{subfigure}[t]{0.5\columnwidth}
      \centering
      \includegraphics[width=\textwidth]{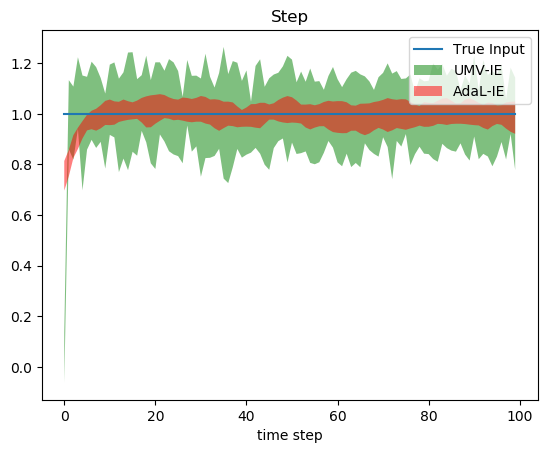}
  \end{subfigure}%
  ~ 
  \begin{subfigure}[t]{0.5\columnwidth}
      \centering
      \includegraphics[width=\textwidth]{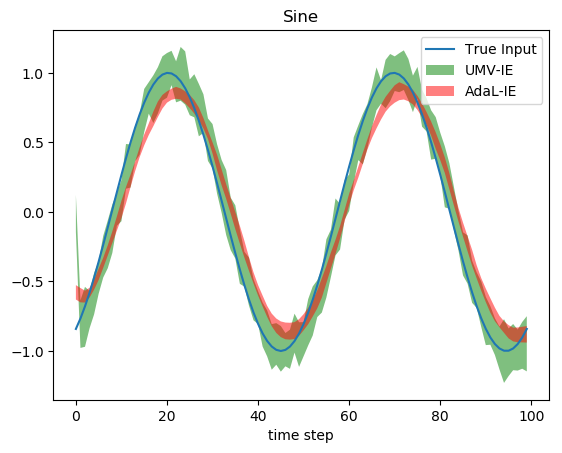}
  \end{subfigure}%
  \\
  \begin{subfigure}[t]{0.5\columnwidth}
      \centering
      \includegraphics[width=\textwidth]{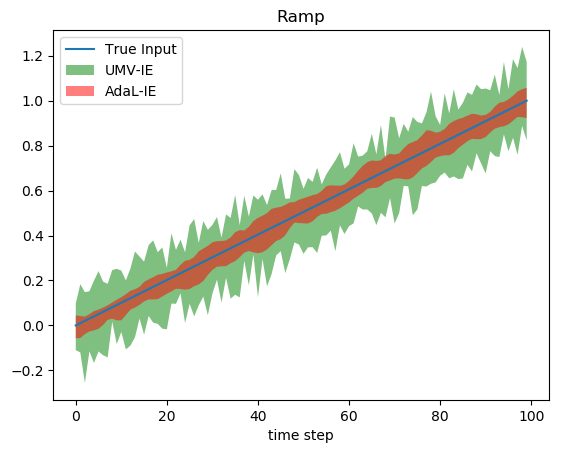}
  \end{subfigure}%
  ~ 
  \begin{subfigure}[t]{0.5\columnwidth}
      \centering
      \includegraphics[width=\textwidth]{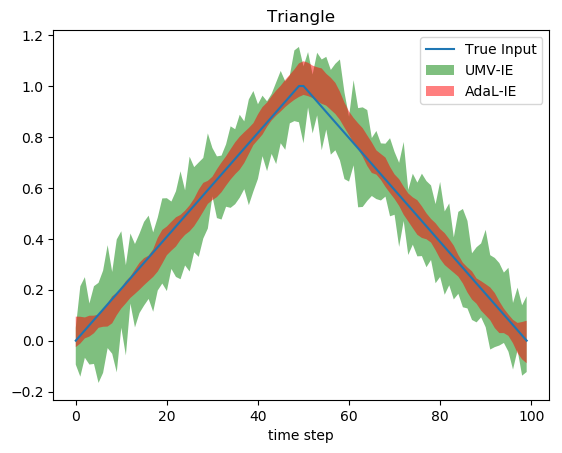}
  \end{subfigure}%
  \\
  \begin{subfigure}[t]{0.5\columnwidth}
      \centering
      \includegraphics[width=\textwidth]{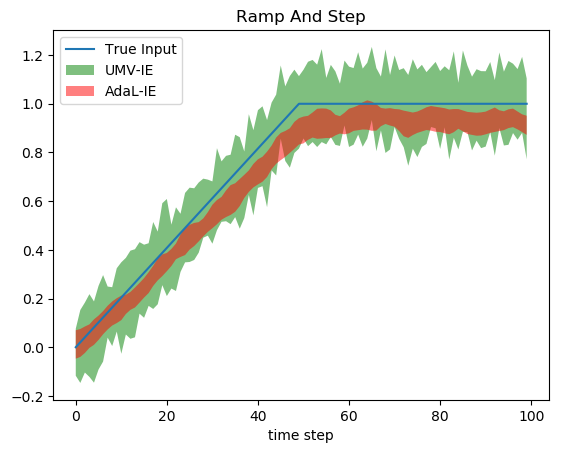}
  \end{subfigure}%
  ~ 
  \begin{subfigure}[t]{0.5\columnwidth}
      \centering
      \includegraphics[width=\textwidth]{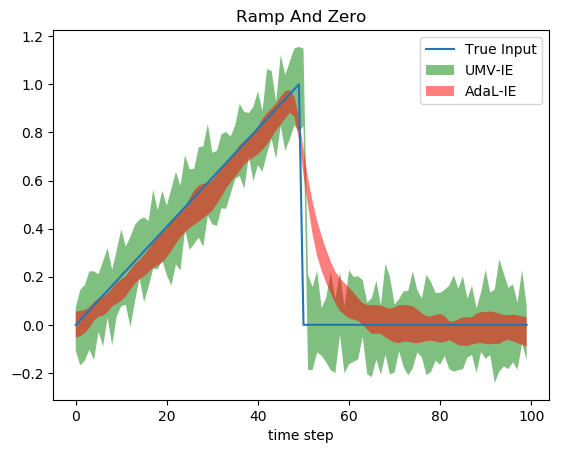}
  \end{subfigure}%
  \\
  \begin{subfigure}[t]{0.5\columnwidth}
      \centering
      \includegraphics[width=\textwidth]{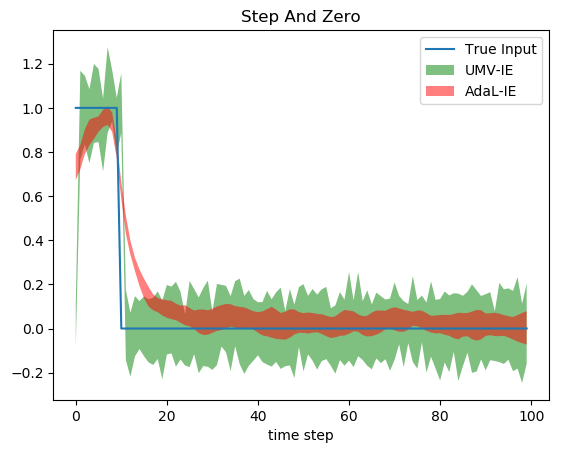}
  \end{subfigure}%
  ~ 
  \begin{subfigure}[t]{0.5\columnwidth}
      \centering
      \includegraphics[width=\textwidth]{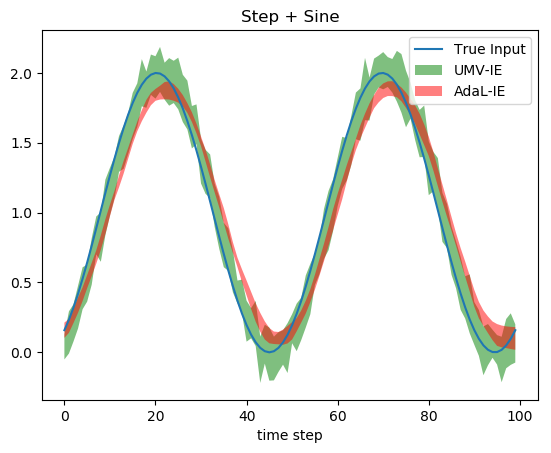}
  \end{subfigure}%
  \caption{Input Signal Estimation in the Spring-Mass System. We can observe that the variance of the AdaL-IE estimates is lower than those of UMV-IE. The trade-off comes at cost of some bias, which can also be seen in the plots. Finally, the RMS error of AdaL-IE is smaller than that of UMV-IE.} \label{fig:IE_SpringMass}
\end{figure}
\vspace{-0.4em}
}

{\setlength{\tabcolsep}{5pt}
\begin{table}[htpb]
\centering
\scriptsize 
\caption{Mean RMS Estimation Error. In each row a different input signal is applied to each system in the columns. The top sub-row has the error of the AdaL-IE estimates while the bottom sub-row the UMV-IE estimates.} \label{tab:experimenterror}
\begin{tabular}{cccccccc}
  \hline \hline
  & Spring & Double & Email & HTTP & Rand & NonLin & NonLin \\
  & -mass  & int.   & server & server & sys. & 4.27 & 4.61 \\
  \hline 
  \multirow{2}{*}{(a)}  & \bm{$0.061$}    & \bm{$0.049$}  & \bm{$0.112$}  & \bm{$0.234$}  & \bm{$0.049$}  & \bm{$0.237$}  & \bm{$0.237$}  \\
                        & $0.191$         & $0.187$       & $0.325$       & $0.885$       & $0.188$       & $1.673$       & $1.673$       \\ \hline
  \multirow{2}{*}{(b)}  & \bm{$0.544$}    & $0.194$       & \bm{$0.139$}  & \bm{$0.318$}  & $0.375$       & \bm{$0.500$}  & \bm{$0.500$}  \\
                        & $1.626$         & $0.202$       & $0.334$       & $0.887$       & \bm{$0.228$}  & $1.678$       & $1.677$       \\ \hline
  \multirow{2}{*}{(c)}  & \bm{$0.584$}    & \bm{$0.044$}  & \bm{$0.123$}  & \bm{$0.241$}  & \bm{$0.085$}  & \bm{$0.257$}  & \bm{$0.263$}  \\
                        & $1.666$         & $0.162$       & $0.317$       & $0.898$       & $0.166$       & $1.674$       & $1.677$       \\ \hline
  \multirow{2}{*}{(d)}  & \bm{$0.063$}    & \bm{$0.055$}  & \bm{$0.113$}  & \bm{$0.255$}  & \bm{$0.056$}  & \bm{$0.285$}  & \bm{$0.259$}  \\
                        & $0.162$         & $0.162$       & $0.310$       & $0.869$       & $0.161$       & $1.677$       & $1.681$       \\ \hline
  \multirow{2}{*}{(e)}  & $0.073$         & \bm{$0.049$}  & \bm{$0.111$}  & \bm{$0.239$}  & \bm{$0.062$}  & \bm{$0.261$}  & \bm{$0.266$}  \\
                        & \bm{$0.021$}    & $0.161$       & $0.312$       & $0.859$       & $0.156$       & $1.681$       & $1.645$       \\ \hline
  \multirow{2}{*}{(f)}  & \bm{$0.584$}    & \bm{$0.102$}  & \bm{$0.131$}  & \bm{$0.267$}  & \bm{$0.163$}  & \bm{$0.297$}  & \bm{$0.307$}  \\
                        & $1.640$         & $0.187$       & $0.320$       & $0.889$       & $0.221$       & $1.653$       & $1.650$       \\ \hline
  \multirow{2}{*}{(g)}  & \bm{$0.123$}    & \bm{$0.120$}  & \bm{$0.137$}  & \bm{$0.274$}  & $0.259$       & \bm{$0.290$}  & \bm{$0.294$}  \\
                        & $0.217$         & $0.217$       & $0.345$       & $0.910$       & $0.277$       & $1.651$       & $1.676$       \\ \hline
  \multirow{2}{*}{(h)}  & \bm{$0.579$}    & $0.176$       & \bm{$0.140$}  & \bm{$0.305$}  & $0.533$       & \bm{$0.533$}  & \bm{$0.515$}  \\
                        & $1.638$         & $0.183$       & $0.324$       & $0.897$       & \bm{$0.222$}  & $1.676$       & $1.701$       \\ \hline

  \hline 
\end{tabular}
\vspace{-2em}
\end{table}
}
    
When the input signal changes abruptly, such as (f) and (g), Assumption~\ref{assum:input_functions} is violated. 
In these cases, the bias around discontinuity points is large, the total error is larger than with UMV-IE, and a delay appears in the estimated input signal.  
  
\textbf{Sparsity of estimator parameters:} 
In Fig.~\ref{fig:Alphas} we show the coefficients of some of \nth{1}, \nth{50}, and \nth{98} input filters and observe that only a small number of indexes around $t$ to have non-zero entries. 
In Fig.~\ref{sfig:nonZero} we show the count of how many non-zero parameters each estimator has for all $t$. 

The bias term and the simplex geometry induces the sparsity pattern. 
The further away the indexes $\tau$ are from the estimation point $t$, the larger $q_t[\tau]$ is leading to larger values of the objective in \eqref{eq:relaxed_optimization}. The simplex geometry is the dual of $l$-1 regularization, which is known to be sparsity inducing. 
For example, for estimating $u_{50}$, only $\alpha_{50}[49:53]$ are non-zero. 
This coincides with the intuition that $u_{50}$ only affects current and future outputs, but later outputs have exponentially less information. 
However, $\alpha_{50}[49]$ is also non-zero. 
Inputs $u_{49}$ and $u_{50}$ are close due to Lipschitz continuity (Assumption~\ref{assum:input_functions}), hence $y_{49}$ also contains information about $u_{50}$. 

The sparsity pattern also suggests that this estimator can be used (almost) online. Rather than waiting for $T$ measurements, we only have to wait for 4 or 5 to estimate the input. 
    
\begin{figure}[htpb]
  \centering
  \begin{subfigure}[t]{0.5\columnwidth}
      \centering
      \includegraphics[width=\textwidth]{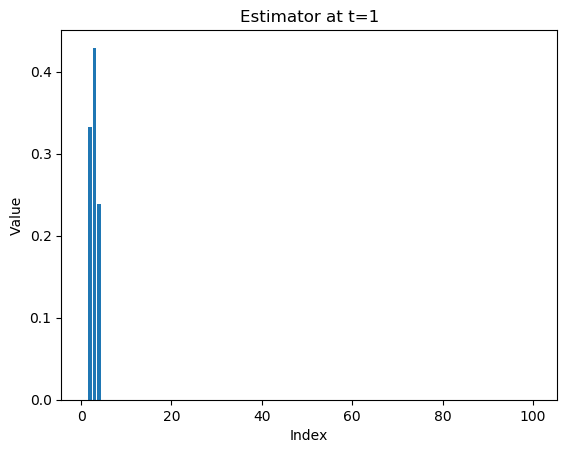}
      \caption{Parameters At $t=1$}\label{sfig:Alphas1}
  \end{subfigure}%
  ~ 
  \begin{subfigure}[t]{0.5\columnwidth}
      \centering
      \includegraphics[width=\textwidth]{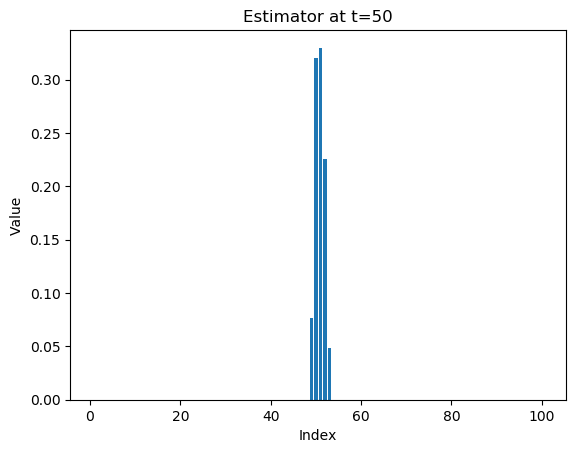}
      \caption{Parameters At $t=50$} \label{sfig:Alphas50}
  \end{subfigure}%
  ~ \\
  \begin{subfigure}[t]{0.5\columnwidth}
      \centering
      \includegraphics[width=\textwidth]{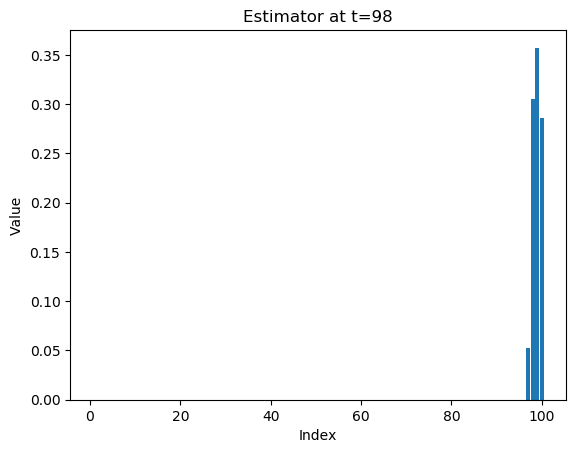}
      \caption{Parameters At $t=98$} \label{sfig:Alphas98}
  \end{subfigure}%
  ~
  \begin{subfigure}[t]{0.5\columnwidth}
      \centering
      \includegraphics[width=\textwidth]{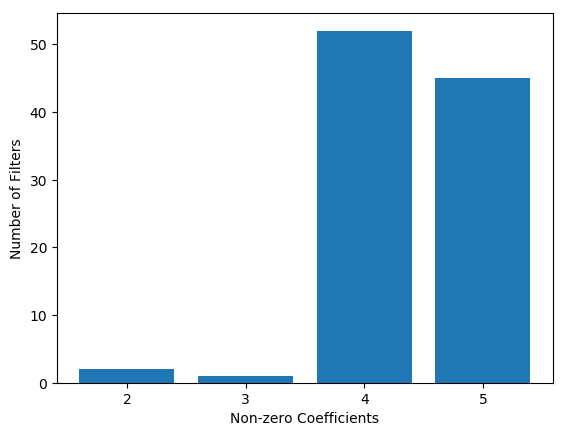}
      \caption{Non-Zero Coefficients} \label{sfig:nonZero}
  \end{subfigure}%
  \caption{Sparsity Pattern Of Estimator Parameters. At every time step, the estimator parameters have all non-zero coefficients around this time step. The inherent locality of the solution is due to the bias-variance trade-off.} \label{fig:Alphas}
  \vspace{-1em}
\end{figure}
  
\textbf{Noise Magnitude Effect:} 
To understand the trade-off between bias and variance of the estimator, we analyze the spring-mass system with different noise magnitudes. 
We show the results in Fig.~\ref{fig:Noise}. 
When the noise is small compared to the magnitude of the signal as in Fig.~\ref{sfig:NoiseSmall}, then the bias is larger than the variance reduction and UMV-IE performs better. 
On the other hand, when the noise has the same magnitude as the signal, then the bias is meaningless as shown in Fig.~\ref{sfig:NoiseLarge}, and AdaL-IE outperforms UMV-IE. 
Even for cases in which the noise is smaller than the signal, the variance reduction is considerable, as shown in Fig.~\ref{fig:IE_SpringMass}. 
\begin{figure}[htpb]
  \centering
  \begin{subfigure}[t]{0.5\columnwidth}
      \centering
      \includegraphics[width=\textwidth]{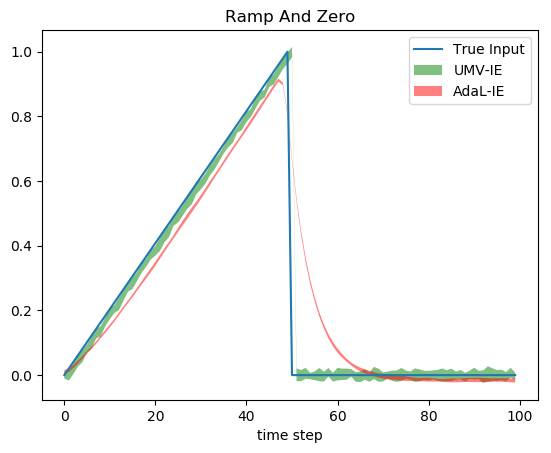}
      \caption{Noise $b=0.01$} \label{sfig:NoiseSmall}
  \end{subfigure}%
  ~ 
  \begin{subfigure}[t]{0.5\columnwidth}
      \centering
      \includegraphics[width=\textwidth]{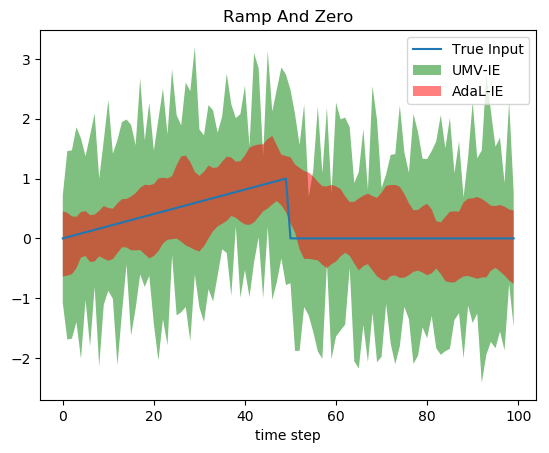}
      \caption{Noise $b=1$}  \label{sfig:NoiseLarge}  
  \end{subfigure}%
  \caption{For settings with high-noise to signal ratio, the variance reduction obtained with AdaL-IE overcomes the bias in the estimator. On the other hand, when the noise is low, UMV-IE estimates have lower errors; the variance reduction does not compensate for the bias increase.} \label{fig:Noise}
  \vspace{-2em}
\end{figure}

\subsection{Learning-from-Observations} \label{ssec:experiments:UL}
After analyzing the quality of the data that the estimator recovers, the question that we answer in this section is if this data is \textit{good enough} to learn a controller that successfully executes a given task. 
As a benchmark, we consider the swing-up and stabilization of the inverted pendulum task of the openAI gym \cite{brockman2016openai} with randomly perturbed parameters (10 \% of the value).
We generate the \textit{expert} trajectories with feedback linearizing controller, $f(\phi, \omega) = 15 \sin(\phi) + 30 \phi + 8 \omega $, where $\phi$ is the angular position and $\omega$ the angular velocity of the pendulum. 
Additively to the input created by the expert, we simulate Uniform process and measurement noise. 
From the demonstrations, only noisy measurements of the angular position and velocity are available.
A task is successful if the resulting controller stabilizes the pendulum in the upright position after 100 seconds. 
We compare three different controller architectures, a linear function, a neural network with 2 hidden layers of 5 neurons with ReLU activation layers, and a neural network with the same architecture but sigmoid activation layers. 

We split this LfO problem into two steps: input estimation and imitation learning with the estimated inputs. 
For the former, we use both UMV-IE and AdaL-IE. 
For the latter we use Behavioral Cloning \cite{bain1999framework}. More advanced algorithms, such as DAGGER \cite{ross2011reduction} or IRL \cite{abbeel2004apprenticeship}, are left for future work. 

In Table~\ref{tab:controllers} we show how many stabilizing controllers we learned this way. 
The original controller can stabilize 930 out of 1000 demonstrations because it is not good enough for the swing-up task from \emph{all} initial conditions due to actuator saturation.
The linear controller performs even worst because it does not account for non-linearities.
The sigmoid activations architecture performs better than the linear, but it can not replicate the original controller exactly due to a mixture between noisy targets, insufficient data, and approximation error.
More work is needed to understand how having more data affects the quality of the input estimation targets and its interaction with LfO. 
The ReLU activations architecture performs worst than the linear because it has very sharp boundaries around the equilibrium, and large noise signals drive the system quickly out of stability. 

To understand the difference between AdaL-IE and UMV-IE targets, we look at the closed-loop eigenvalues of the linearized system around the unstable equilibrium for the linear controller. 
When we use AdaL-IE targets, the closed-loop system is stable, while when we use UMV-IE targets, it is closed-loop unstable. 
The intuition for this is that after the swing-up phase, most of the expert trajectories are around the equilibrium. 
In this phase, the noise that drives the system out of the equilibrium and the controller input that drives the system back have about the same magnitude.
The situation is that of Fig.~\ref{sfig:NoiseLarge} and the estimates of AdaL-IE have smaller error than those of UMV-IE. 

{\setlength{\tabcolsep}{7pt}
\begin{table}[htpb]
\centering
\caption{Number of stabilizing controllers out of 1000 demonstrations. The inputs are estimated using AdaL-IE and UMV-IE respectively and used as Behavioral Cloning targets.} \label{tab:controllers}
\vskip 0.1in
\begin{tabular}{ccc}
  \hline \hline
  Controller & AdaL-IE & UMV-IE \\
  \hline 
  Linear Architecture   &  \bm{$496$} & 0\\
  ReLU Activations      &  \bm{$332$} & 0\\
  Sigmoid Activations   &  \bm{$754$} & 224\\ \hline
  Original Controller   &   \multicolumn{2}{c}{930} \\
  \hline \hline
\end{tabular}
\end{table}
\vspace{-1em}
}

\section{CONCLUSION}
In this paper, we presented a new input estimator for a general class of dynamical systems by optimizing a high-probability upper bound of the estimation error. 
This upper bound comes from a natural decomposition in bias and variance of the estimation.  
We showed in experiments that this estimator has lower errors than previous state-of-the-art (UMV-IE) methods. 
Finally, we tested this estimator on the LfO framework and successfully learned stabilizing controllers using AdaL-IE as input estimation block.

\bibliographystyle{IEEEtran}
\bibliography{root}

\end{document}